\documentclass[12pt,a4paper]{article} 
\usepackage{amsfonts, amsmath, amssymb, amsthm}
\usepackage{graphicx}

\usepackage{float}
\usepackage{mathtools}
\usepackage[dvipsnames]{xcolor}

\usepackage{dsfont}

\floatstyle{ruled}
\newfloat{algorithm}{tbp}{loa}
\providecommand{\algorithmname}{Algorithm}
\floatname{algorithm}{\protect\algorithmname}

\newtheorem{thm}{Theorem}[section]
\newtheorem{lemma}[thm]{Lemma} 
\newtheorem{prop}[thm]{Proposition} 
\newtheorem{cor}[thm]{Corollary}
\newtheorem{assumption}[thm]{Assumption}

\newtheorem{defn}[thm]{Definition} 
\newtheorem{rem}[thm]{Remark}

\usepackage{natbib}
\usepackage{hyperref}       
\hypersetup{colorlinks,linkcolor={blue},citecolor={blue},urlcolor={red}}

\usepackage{algorithmicx}

\algnewcommand\algorithmicinput{\textbf{INPUT:}}
\algnewcommand\INPUT{\item[\algorithmicinput]}
\algnewcommand\algorithmicoutput{\textbf{OUTPUT:}}
\algnewcommand\OUTPUT{\item[\algorithmicoutput]}
\algnewcommand\algorithmicparameter{\textbf{PARAMETER:}}
\algnewcommand\PARAMETER{\item[\algorithmicparameter]}

\algnewcommand\algorithmicPreprocessing{\textbf{PREPROCESSING:}}
\algnewcommand\PREPROCESSING{\item[\algorithmicPreprocessing]}

\DeclareMathOperator*{\plim}{plim}

\DeclareMathOperator*{\Bcal}{\mathcal{B}}
\DeclareMathOperator*{\obs}{\text{obs}}
\DeclareMathOperator*{\Fcal}{\mathcal{F}}
\DeclareMathOperator*{\Mcal}{\mathcal{M}}

\DeclareMathOperator*{\Pcal}{\mathcal{P}}
\DeclareMathOperator*{\Ocal}{\mathcal{O}}
\DeclareMathOperator*{\Gcal}{\mathcal{G}}
\DeclareMathOperator*{\Ical}{\mathcal{I}}
\DeclareMathOperator*{\Acal}{\mathcal{A}}

\DeclareMathOperator*{\PP}{\mathbb{P}}
\DeclareMathOperator*{\NN}{\mathbb{N}}
\DeclareMathOperator*{\EE}{\mathbb{E}}
\DeclareMathOperator{\RR}{\mathbb{R}}
\DeclareMathOperator*{\rank}{\textup{rank}}

\usepackage[english]{babel}

\begin{document}

\thispagestyle{empty} {\vspace*{-2.8cm} \hspace*{7.5cm}
\includegraphics[width=75mm,viewport=0 0 235
85,clip]{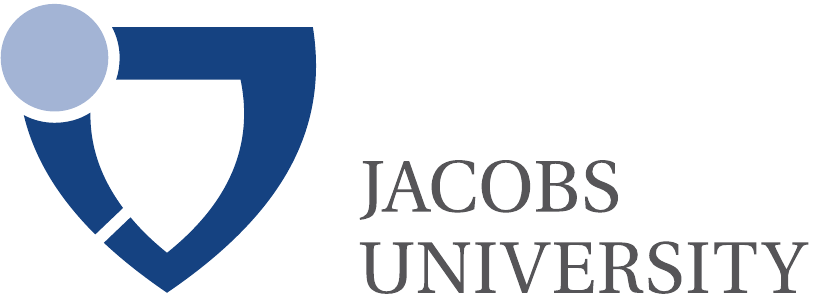}} \vspace{6.5cm}

{\noindent \large \textsf{Tianlin Liu}} \\ \\

{\noindent \Large \bf  A Consistent Method for Learning OOMs from Asymptotically Stationary Time Series Data Containing Missing Values} \\ \vfill
{\noindent \LARGE \textsf{Technical Report No. 38}} \\
{\noindent \textsf{\today}} \\ {\hspace*{-3.17cm}
\rule[3mm]{\textwidth}{0.75pt}} \\ {\LARGE \textsf{Department of Computer Science and Electrical Engineering}} \\

\newpage \thispagestyle{empty}

{ 
\noindent 
\Large {\bf A Consistent Method for Learning OOMs from Asymptotically Stationary Time Series Data Containing Missing Values} \\

\normalsize 

\vspace{1cm} 

\noindent 
{\bf Tianlin Liu}\\ \\ {\it Jacobs University Bremen\\ Department of Computer Science and Electrical Engineering\\ Campus Ring \\ 28759 Bremen\\ Germany\\ \\
E-Mail: t.liu@jacobs-university.de\\
}

\section*{Abstract}

In the traditional framework of spectral learning of stochastic time series models, model parameters are estimated based on trajectories of fully recorded observations. However, real-world time series data often contain missing values, and worse, the distributions of missingness events over time are often not independent of the visible process. Recently, a spectral OOM learning algorithm for time series with missing data was introduced and proved to be consistent, albeit under quite strong conditions. Here we refine the algorithm and prove that the original strong conditions can be very much relaxed. We validate our theoretical findings by numerical experiments, showing that the algorithm can consistently handle missingness patterns whose dynamic interacts with the visible process.

\newpage
\thispagestyle{empty}
\tableofcontents


\newpage
\setcounter{page}{1}

\section{Introduction}

Spectral methods have become widely used to model probabilistic grammars \citep{Bailly2010, Balle2012, Cohen2012, Cohen2013, Balle2015}, stochastic processes \citep{Hsu2012, Anandkumar2012, Jordan2013, Anandkumar2014, Thon2015, Wu2016}, and controlled dynamical systems \citep{Boots2011b, Hamilton2014, Hefny2015, Azizzadenesheli2016, Hefny2018}. Compared with likelihood-based methods such as Expectation Maximization (EM), spectral learning methods have two appealing properties: (i) they are consistent learning methods with convergence guarantees, and (ii) they are, in principle, non-iterative learning methods which are computationally inexpensive. These properties have made spectral learning a promising tool for learning and analyzing dynamical systems.

Stochastic time series modeling is one of the application areas for spectral methods. In the traditional framework of spectral learning, the training sequences of observations are assumed to contain no missing values. However, this assumption can be violated in real-world data. For instance, in longitudinal studies of disease treatment \citep{Hedeker2006}, the disease status of patients are often intermittently missing due to patients' skipped visits; in gene expression analysis, the gene data generated by microarray experiments often contain missing expression values \citep{Troyanskaya2001}; in affective computing for emotion measurement, the cognitive-affective states of subjects are usually sparsely annotated by human experts \citep{Grafsgaard2011}, where the un-annotated timestamps can be regarded as missing values.

There exist multiple methods to learn stochastic time series models from training data containing missing values. One obvious way is to assemble shorter trajectories that are free from missing values as the new training data. However, this approach might suffer from substantial information loss. Another way is to design algorithms that acknowledge the missing values in the training data. In this line of efforts, likelihood-based methods such as EM algorithms for Hidden Markov Models have been investigated in previous research \citep{Yeh2012, Yu2003}. Nonetheless, similar to other EM-based algorithms, these algorithms rely on local search heuristics, giving rise to locally optimal results and costly computation. 

\citet[Chapter 7]{Thon2017} recently presented a spectral learning algorithm that admits missing values in the training data. In this novel approach, one first estimates an Input-Output Observable Operator Model from missingness-observation sequences using a spectral method, and then reduces it to an Observable Operator Model, which describes the underlying stochastic process. Here we first give a condensed yet self-contained introduction to the algorithm first proposed in \citep[Chapter 7]{Thon2017}; we then analyze the theoretical properties of the algorithm by (i) presenting and analyzing a modified frequency estimator that acknowledges missing values in training data, and (ii) determining the consistency of the proposed spectral algorithm under  much more relaxed set of assumptions than in \citep[Chapter 7]{Thon2017}. We provide numerical experiments to demonstrate the capabilities of the proposed algorithm and our theoretical findings. 

\subsection{Notation}

Let $\Sigma_{\mathcal{A}}$ and $\Sigma_{\mathcal{O}}$ be alphabets for actions and observations of a system. We use the symbol $a \in \Sigma_{\mathcal{A}}$ to denote an action, and the symbol $o \in \Sigma_{\mathcal{O}}$ to denote an observation. We use a symbol with a bar to denote a word (a sequence of symbols), e.g., $\bar{o}$, and we use the lower index if the starting and ending time are specified for the word, e.g., $\bar{o}_{1:N} = o_1 \cdots o_N$. Let $\Sigma_{\Ocal}^\ast$ be the set of words over $\Sigma_{\Ocal}$, and let  $\Sigma_{\Ocal}^l$ be the set of words over $\Sigma_{\Ocal}$ with length $l$. The upper index with square brackets is reserved to count the trajectories of words, e.g., $\bar{x}_{1:N}^{[j]}$ forms the $j$-th trajectory. In the same fashion, a sequence of action-observation pairs is denoted by $\bar{a} \bar{o}$, and if the starting and ending time are specified, $\bar{a}_{1:N} \bar{o}_{1:N}$. Let $(\Sigma_{\Acal} \times \Sigma_{\Ocal})^\ast$ be the set of words over the set $\Sigma_{\Acal} \times \Sigma_{\Ocal}$, where $\times$ denotes the operation of cartesian product.


For a matrix $M$, we use $M^\top$ to denote the matrix transpose, $M^{-1}$ to denote the matrix inverse, $M^\dagger$ to denote the Moore-Penrose pseudo-inverse, and $[M]_{i,j}$ for the entry in the row indexed by $i$ and column indexed by $j$, $[M]_j$ for the column indexed by $j$, $[M^\top]_j^\top$ for the row indexed by $j$. Throughout this manuscript we use $\NN$ to denote the set of positive integers. We write $[n] =  \{1, 2, \cdots, n\}$ for some $n \in \NN$.  A matrix with all entries 0 will be denoted by $\mathbf{0}$. 

We denote the probability of an event by $\PP(\cdot)$ and the probabilistic condition by $\mid$. The probability limit of a sequence $s_n$, if such a limit exists, is denoted by $\plim s_n$. The convergence in probability is denoted by $\xrightarrow[]{p}$.

\section{Background} \label{sec:Background}
In this section, we review the basic definitions for stochastic processes and dynamical systems by heavily reusing \citep{Bauer1972, Schoenhuth2006, Schoenhuth2008} sometimes verbatim. We then take a brief overview of Observable Operator Models and Input-Output OOMs under the framework of Sequential Systems, which have been systematically introduced in \citep{Thon2015}. 

A discrete-time, finite-valued stochastic process is a quadruple $(\Omega, \Fcal, \PP, (X_t)_{t \in \NN})$, in which $(\Omega, \Fcal, \PP)$ is a probability space and $(X_t)_{t \in \NN}$ is a family of random variables on this probability space taking values in a measurable space $(S, \Pcal)$, where $S$ is a finite set and $\Pcal$ is the power set of $S$. We write $S^{\NN} = \prod_{i \in \NN} S_i$, where every factor $S_i$ is equal to $S$. That is, $S^{\NN}$ is the collection of $S$-valued right-infinite sequences. We let $\Gcal$ be the $\sigma$-algebra over $S^{\NN}$ generated by the cylinder sets of sequences in $S^{\NN}$. Further, we let $\otimes_{t \in \NN} X_t$ be the product random variable, which is a map from $\Omega$ to $S^{\NN}$. Let $\mu$ be the distribution of the random variable $\otimes_{t \in \NN} X_t$.

We define a \emph{left shift} transformation $T: S^{\NN} \to S^{\NN}$ by $T (s_1, s_2, s_3, \cdots ) = (s_2, s_3, \cdots )$ for all $(s_1, s_2,  s_3, \cdots ) \in S^{\NN}$. For a finite-length sequence $(s_1, s_2, \cdots, s_t) \in S^t$ for $t \in \NN$, we write $T(s_1, s_2, \cdots, s_t) = (s_2, \cdots, s_t) \in S^{t-1}$. For a set of sequences $B \in \Gcal$, we define $T(B) =  \{ T( s_1, s_2, s_3, \cdots) \mid (s_1, s_2, s_3, \cdots)  \in B \}$. The transformation $T^i$ for $i \in \NN$ is defined as the $i$-times composition of $T$, i.e., $T^i \coloneqq \underbrace{T \circ T \circ \dots \circ T}_{i\text{-times}}$. Similarly, for all $B \in \Gcal$ and $i \in \NN$, we define the $i$-times \emph{right shift} of a set of sequences $B$ by setting $T^{-i} (B) \coloneqq \{(s_1 \cdots s_i s_{i+1} s_{i+2} \cdots) \mid (s_{i+1} s_{i+2} \cdots) \in B, s_1 \cdots s_i \in S^i\}$. 

Given a stochastic process $(\Omega, \Fcal, \PP, (X_t)_{t \in \NN})$, the quadruple $(S^{\NN}, \Gcal, \mu, T)$ is called the induced \emph{canonical dynamical system}, which exists and is uniquely defined \citep[Definition 4.2]{Schoenhuth2006} \citep[Corollary 12.1.4]{Bauer1972}. In this manuscript, we only work with such canonical dynamical systems induced by stochastic processes, so when we talk about the properties of a dynamical system (e.g., stationarity), we also refer to these properties of the corresponding stochastic process.

A dynamical system is said to be {\em stationary} (relative to $T$), if $\mu(B) = \mu(T^{-1} B)$ for all $B \in \Gcal$; a dynamical system is called {\em asymptotically stationary} (relative to $T$), if there is a measure $\bar{m}$ such that $\lim_{i \to\infty} \mu(T^{-i}B)  = \bar{m}(B)$ for all $B \in \Gcal$, where the measure $\bar{m}$ is called the asymptotically stationary measure; a dynamical system is called {\em asymptotically mean stationary} (AMS) (relative to $T$), if there is a measure $\bar{\mu}$ such that $\lim_{n\to\infty}\frac{1}{n}\sum_{i=0}^{n-1} \mu(T^{-i}B)  = \bar{\mu}(B)$ for all $B \in \Gcal$, where the measure $\bar{\mu}$ is called the AMS measure. In this manuscript, we will be considering only asymptotically stationary dynamical systems, and in this case the asymptotically stationary measure is the same as AMS measure.

Given a dynamical system $(S^{\NN}, \Gcal, \mu, T)$, a function $g: S^{\NN} \to \RR$ is said to be a \emph{measurement} of the dynamical system if $g$ is $\Gcal$-$\Bcal(\RR)$ measurable, where $\Bcal(\RR)$ is the Borel $\sigma$-algebra of $\RR$. The dynamical system is said to be \emph{ergodic with respect to the measurement} $g$ if the sample average $\frac{1}{n} \sum_{i = 0}^{n-1} g(T^i \bar{s})$ converges as $n \to \infty$ for almost all $\bar{s} \in S^{\NN}$. An event $I\in \Gcal$ is called \emph{invariant} (relative to $T$), if $T^{-1}I = I$.  The set of invariant events is a sub-$\sigma$-algebra of $\Gcal$ which we will denote by $\Ical$.  A dynamical system is said to be \emph{ergodic} (relative to $T$), if $\mu(I)\in\{0,1\}$ for any such invariant event $I\in\Ical$.

\begin{prop} \textup{(Corollary 7.2.1. of \citep{Gray2009})} \label{prop:Ergodic}
If a dynamical system is ergodic, AMS with stationary measure $\bar{\mu}$, and the sequence $n^{-1} \sum_{i=0}^{n-1} g T^i$ is uniformly integrable with respect to $\bar{\mu}$, where $g$ is a measurement of the dynamical system, then the following limit is true $\mu$-a.e., $\bar{\mu}$-a.e., and in $L^1(\mu)$:

\[ \lim_{n \to \infty} \frac{1}{n} \sum_{i = 0}^{n-1} g T^i = \bar{\EE} (g), \]
where $\bar{\EE}(\cdot)$ denotes the expectation with respect to the AMS measure $\bar{\mu}$ and $L^1(\mu)$ is the space of all $\mu$-integrable functions.
\end{prop}

\subsection{Sequential Systems, OOMs, and IO-OOMs}

We now define the Sequential Systems, which are abstract linear algebraic models originally proposed to study Stochastic Finite Automata \citep{Carlyle1971}. 

 \begin{defn} \textup{(Sequential System)}.
 A $d$-dimensional linear \emph{Sequential System} (SS) over the alphabet $\Sigma$ is a structure $\mathcal{M} = ( \sigma, \{\tau_z \}_{z \in \Sigma}, \omega_{\epsilon} )$, where
$\sigma$ is a linear evaluation function $\mathbb{R}^d \to \mathbb{R}$, each $\tau_z \in \mathbb{R}^{d \times d}$ is a linear operator, and $\omega_{\epsilon} \in \mathbb{R}^d$ is the initial state.

For a SS $\mathcal{M}$, its \emph{external function} $ f_{\mathcal{M}}$ is defined by
 \[ f_{\mathcal{M}}: \Sigma^{\ast} \to \mathbb{R}:  \quad  f_{\mathcal{M}}(x_1 \cdots x_n ) \coloneqq \sigma \tau_{x_n} \cdots \tau_{x_1} \omega_{\epsilon} \]
where $x_1 \cdots x_n \in \Sigma^{\ast}$.  
 \end{defn}
 
 We regard two SSs as equivalent if they describe the same external function $f$.
 
 \begin{defn} \textup{(Equivalent SSs)} \label{defn:equivalentSSs}
 Two SSs $\Mcal$ and $\Mcal^\prime$ are equivalent, denoted by $\Mcal \simeq \Mcal^\prime$, if they define the same external function, i.e., if $f_{\Mcal} = f_{\Mcal^\prime}$.
 \end{defn}
 
 Based on the above definition of equivalence, it is clear that two SSs are equivalent if they are subject to a similarity transformation. 
 
\begin{lemma} \textup{(\citep[Lemma 10]{Thon2015})} \label{lemma:similarSS}
 Let $\Mcal =  ( \sigma, \{\tau_z \}_{z \in \Sigma}, \omega_{\epsilon} )$ be a $d$-dimensional SS, and $\rho \in \RR^{d \times d}$ be non-singular. Then $\Mcal  \simeq \Mcal^\prime$, where
$\Mcal^\prime = (\sigma \rho^{-1}, \{ \rho \tau_z \rho^{-1}\}_{z \in \Sigma}, \rho \omega_{\epsilon} )$. 
 \end{lemma}
 
 A SS could be further specified as a Stochastic Multiplicity Automaton (SMA), an Observable Operator Model (OOM), or an Input-Output OOM (IO-OOM), depending on whether one is interested in modeling probabilistic languages, stochastic processes, or controlled processes. We proceed to define OOMs.

 \begin{defn} \textup{(OOM)}.
 An \emph{uncontrolled process} over the alphabet $\Sigma_{\Ocal}$ is a function $f : \Sigma_{\Ocal}^\ast \to [0,1]$ that satisfies (i) $f(\epsilon) = 1$ and (ii) for all $\bar{x} \in \Sigma_{\Ocal}^\ast: f(\bar{x}) = \sum_{o \in \Sigma_{\Ocal}} f(\bar{x} o)$. An \emph{Observable Operator Model} (OOM) is a SS that models an uncontrolled process.  
 
 \end{defn}

We see that OOMs are defined by letting external functions of SSs to be uncontrolled processes. Similarly, we could define Input-Output OOMs, or equivalently\footnote{ IO-OOMs and PSRs have different formalisms, but using the formulation in Definition \ref{defn:IO-OOM}, IO-OOMs are equivalent to PSRs \citep{Thon2015}. We use IO-OOMs in this report instead of PSRs only for the consistency in notations. Note the original definition of IO-OOMs of \citep{Jaeger1998} is by now deprecated, with which IO-OOMs were more restrictive than PSRs \citep{Singh2004}.}, Predictive State Representations (PSRs), by setting the outer functions of SSs to be controlled processes.

\begin{defn} \label{defn:IO-OOM} \textup{(IO-OOM)} A \emph{controlled process} over the alphabet $\Sigma_{\Acal} \times \Sigma_{\Ocal}$ is a function $f: (\Sigma_{\Acal} \times \Sigma_{\Ocal})^\ast \to [0,1]$ that satisfies (i) $f(\epsilon) = 1 $ and (ii) $\forall \bar{x} \in (\Sigma_{\Acal} \times \Sigma_{\Ocal})^\ast, a \in \Sigma_{\Acal}: f(\bar{x})= \sum_{o \in \Sigma_{\Ocal}}(\bar{x}ao)$. An \emph{Input-Output OOM} (IO-OOM) is a SS that models a controlled process.
\end{defn}

In general, OOMs and IO-OOMs are models with \emph{predictive states}, meaning that their states encode the necessary information for predicting the future. For this reason, conceptually OOMs and IO-OOMs are very different from models with \emph{latent states} such as HMMs \citep{Bengio1999} and POMDPs \citep{Kaelbling1998}, where states are defined by probability distributions over hidden variables. It has been shown that OOMs and IO-OOMs have greater representational capacity: OOMs extend HMMs \citep{Jaeger2000} and IO-OOMs extend POMDPs \citep{Littman2001}.

\section{Spectral Learning for OOMs from Data Containing Missing Values} \label{section:SpectralMissingValues}

The standard spectral learning algorithms for time series models require that the training sequences be fully recorded. This requirement, however, can be violated in real-world sequential data where missing values are not uncommon. In this section, based on \citep{Thon2017}, we present and analyze a spectral learning algorithm that acknowledges the missing values in the training data, and use such data to learn OOMs which describe the underlying stochastic processes.

\subsection{The Types of Missingness in Time Series Data} \label{subsection:types}

In this subsection, we review the basic definitions for spectral learning with missing values as introduced in \citep{Thon2017}, sometimes using his wording. Consider a stochastic process $(X_t)_{t \in \mathbb{N}}$ that takes values in $\Sigma_{\Ocal}$. We will call this stochastic process the \emph{underlying stochastic process}. Throughout this manuscript we will only deal with underlying stochastic processes that are asymptotically stationary (and therefore AMS). Let $\bar{x}_{1:N} = x_1 \cdots x_N$ be an initial sample from the underlying stochastic process $(X_t)_{t \in \mathbb{N}}$. In practice, for some time steps $t$ we do not observe the value $x_t$, and in this situation we say that $x_t$ is missing. We let $\bar{m}_{1:N} = m_1 \cdots m_N \in \{0, 1\}^N$ be a sequence of \emph{missingness}, with $m_t = 1$ if the value $x_t$ is missing, else $m_t = 0$. Let $\bar{o}_{1:N} \in (\Sigma_{\Ocal} \cup \{\emptyset\} )^N$ denote the sequence of observations of length $N$, where $o_t = x_t$ if $m_t = 0$, i.e., if the observation at time $t$ is not missing, and $o_t = \emptyset$ otherwise. We can pair up $m_t o_t$ for all $t \in [N]$ as a \emph{missingness-observation} sequence, such that $\bar{m}_{1:N} \bar{o}_{1:N} = m_1 o_1 \cdots m_N o_N$ is the initial sample of a missingness-observation process $(M_t O_t)_{t \in \mathbb{N}}$. We use an example to illustrate these notations: For $\Sigma_{\Ocal} = \{\texttt{a,b,c}\}$ and $N = 4$, let $\bar{x}_{1:N} = \texttt{abcc}$ be the underlying sequence, and suppose the first symbol \texttt{a} is missing, then $\bar{m}_{1:N} = \texttt{1000}$, $\bar{o}_{1:N} = \emptyset \texttt{bcc}$, $\bar{m}_{1:N} \bar{o}_{1:N} = \texttt{1}\emptyset\texttt{0b0c0c}$.

For an underlying stochastic process $(X_t)_{t \in \mathbb{N}}$, our goal is to learn a model for $(X_t)_{t \in \mathbb{N}}$ using $\bar{m}_{1:N}\bar{o}_{1:N}$ as training data. To achieve such a goal, we will treat missing values as wildcards or ``don't care'' placeholders for observations (the purpose of which will be clear later).  To have a convenient notation for describing the effects of wildcards, for all $t \in \NN$, we invest an additional random variable $X^{\obs}_{t}: \Omega \to \Sigma_{\Ocal}$, such that $X_t^{\obs}(\omega) \coloneqq X_{t}(\omega)$ for all $\omega \in \Omega$ and for all $t \in \NN$. Additionally, for all $t \in \NN$, we introduce a missing value notation $\emptyset$ upon $X_t^{\obs}$: by writing $X_t^{\obs} = \emptyset$, we simply mean $X_t^{\obs} \in \Sigma_{\Ocal}$, reflecting the wildcard or ``don't care'' placeholder nature of $\emptyset$. This simply means that each $X_t^{\obs}$ is an identical copy of $X_t$ for all $t \in \NN$, with a special missing value notation equipped on the former but not on the latter. 

Given a missingness-observation sequence $\bar{m}_{1:N} \bar{o}_{1:N}$, the joint probability of this missingness-observation sequence is governed by the \emph{missingness  process} $\pi$ and the \emph{observation process} $f$ in the sense that
\begin{multline}\label{eq:combinedProcess}
 \PP \left (M_{1:N } O_{1:N} = \bar{m}_{1:N} \bar{o}_{1:N} \right ) =   \underbrace{\prod_{t = 1}^N \PP \left (  M_t = m_t \mid M_{1:t-1} O_{1:t-1} = \bar{m}_{1:t-1} \bar{o}_{1:t-1} \right )}_{\pi(M_{1:N} O_{1:N} = \bar{m}_{1:N} \bar{o}_{1:N} )} \cdot \\ \underbrace{\prod_{t = 1}^N \PP \left ( O_t = o_t \mid M_{1:t-1} O_{1:t-1} = \bar{m}_{1:t-1} \bar{o}_{1:t-1}, M_t = m_t \right )}_{f(M_{1:N} O_{1:N}= \bar{m}_{1:N} \bar{o}_{1:N} )}.
\end{multline}

If the random variables $M_{1:N } O_{1:N}$ are clear from context, we will simply drop them and write $\PP \left ( \bar{m}_{1:N} \bar{o}_{1:N} \right )$, $f(\bar{m}_{1:N} \bar{o}_{1:N} )$, and $\pi(\bar{m}_{1:N} \bar{o}_{1:N} )$. Note that if $\PP \left ( \bar{m}_{1:N} \bar{o}_{1:N} \right ) = 0$, the factorization at the right hand side of Equation \ref{eq:combinedProcess} would not be defined. For this reason, we assume $\PP \left ( \bar{m}_{1:N} \bar{o}_{1:N} \right ) > 0$ when using the factorization in Equation \ref{eq:combinedProcess}.

We now specify $f(\bar{m}_{1:N} \bar{o}_{1:N} )$ and $\pi(\bar{m}_{1:N} \bar{o}_{1:N} )$ in more detail. Although the random variables $M_t O_t$ for all $t \in \NN$ take values in $\{0, 1\} \times (\Sigma_{\Ocal} \cup \{\emptyset\})$, we are not interested in the pairs $(0, \emptyset)$ and $(1, x)$ for $x \in \Sigma_{\Ocal}$, as they do not make any sense for the underlying stochastic process. For this reason, we impose a restriction on the observation process $f(\cdot)$ by requiring that 
\begin{equation} \label{eq:observationProcess}
\PP ( O_{t} = \emptyset \mid M_{1:t-1} O_{1:t-1} = \bar{m}_{1:t-1} \bar{o}_{1:t-1}, M_t = m_t ) =
\begin{cases}
1& \text{if~} m_t = 1, \\
0& \text{if~} m_t = 0,
\end{cases}
\end{equation}
for all $t$.

We now specify a special case of missingness $\pi (\cdot)$ named AMSAR.

\begin{defn} \textup{(AMSAR missingness).} The values in the stochastic process are said to be \emph{always missing sequentially at random} (AMSAR) if for all $t \in [N]$ and for all $N$ we have:
\[ \PP(M_t = m_t \mid X^{\obs}_{1:t-1} = \bar{o}_{1:t-1}, X_{t:N} = \bar{x}_{t:N}) = \PP(M_t = m_t \mid X^{\obs}_{1:t-1} = \bar{o}_{1:t-1}), \]
where the index $t$ is the current time, $\bar{o}_{1:t-1} \in (\Sigma_{\Ocal} \cup \{\emptyset \})^{t-1}$ is the observation sequence (containing the missing values) prior to the current time $t$, and $\bar{x}_{t:N} \in \Sigma_{\Ocal}^{N-t+1}$ are the values of the underlying stochastic process from the current time $t$ to the future time $N$.
\end{defn}

 Intuitively, AMSAR says that missingness at every time $t$ is conditionally independent of the current and future values of the underlying stochastic process, given the previously observed values. It also says that missingness at time $t$ is independent of which ``true but unobserved'' outputs $o$ have been emitted at times when there was missingness. We consider that this is a realistic assumption for the missing values in real-time sequential data in the sense that the missingness at a time \emph{can} depend on the previous observations.

\begin{lemma}  \label{lemma:AMSAR} \textup{(Lemma 70 of \citep{Thon2017})}
Supposing the missingness is AMSAR, for any $t, \bar{m}, \bar{x}$, we have 
\begin{equation} \label{eq:lemmaAMSAR}
 \PP ( X_{t} = x_t \mid  X^{\obs}_{1:t-1} = \bar{o}_{1:t-1}, M_t = m_t ) = \PP \left ( X_{t} = x_{t} \mid X^{\obs}_{1:t-1} = \bar{o}_{1:t-1} \right ).
\end{equation}
\end{lemma}
\begin{proof} 
We assume $\PP ( X_{t} = x_t,  X^{\obs}_{1:t-1} = \bar{o}_{1:t-1}, M_t = m_t ) > 0$ as otherwise the statement is trivial.
\begin{eqnarray*}
& & \PP \left ( X_{t} = x_t \mid X^{\obs}_{1:t-1} = \bar{o}_{1:t-1}, M_t = m_{t}  \right ) \\
 & = & \PP (M_t = m_{t} \mid X_t = x_t, X^{\obs}_{1:t-1} = \bar{o}_{1:t-1}) \frac{\PP (X_t = x_t,  X^{\obs}_{1:t-1} = \bar{o}_{1:t-1})}{\PP (M_t = m_{t},  X^{\obs}_{1:t-1} = \bar{o}_{1:t-1})} \\
 & = & \PP (M_t = m_{t} \mid X_t = x_t, X^{\obs}_{1:t-1} = \bar{o}_{1:t-1}) \frac{\PP (X_t = x_t \mid  X^{\obs}_{1:t-1} = \bar{o}_{1:t-1})}{\PP (M_t = m_{t} \mid  X^{\obs}_{1:t-1} = \bar{o}_{1:t-1})} \\
& \stackrel{(\ast)}{=}  &  \PP (X_t = x_t \mid X^{\obs}_{1:t-1} = \bar{o}_{1:t-1}) 
\end{eqnarray*}
where $(\ast)$ follows from the AMSAR assumption:
\[ \PP (M_t = m_{t} \mid X_t = x_t,  X^{\obs}_{1:t-1} = \bar{o}_{1:t-1}) = \PP (M_t = m_{t} \mid  X^{\obs}_{1:t-1} = \bar{o}_{1:t-1}).\]

\end{proof}

\begin{prop} \label{prop:probObs}
Let $(X_t)_{t \in \mathbb{N}}$ be an underlying stochastic process and let $(M_t O_t)_{t \in \mathbb{N}}$ be the corresponding missingness-observation process which results from corrupting the underlying stochastic process with an AMSAR missingness. Let $\bar{m}_{1:N} \bar{o}_{1:N}$ be a missingness-observation sequence such that $\PP(\bar{m}_{1:N} \bar{o}_{1:N}) > 0$, then
\[ f(M_{1:N}O_{1:N} =  \bar{m}_{1:N} \bar{o}_{1:N}) = \PP(X^{\obs}_{1:N} = \bar{o}_{1:N}).\]
\end{prop}

\begin{proof}
This proposition was stated in \citep[Equation (18) - (19)]{Thon2017}. 
\begin{eqnarray*}
f(M_{1:N}O_{1:N} =  \bar{m}_{1:N} \bar{o}_{1:N}) & \stackrel{(1)}{=} & \prod_{t \in [N]} \PP \left ( O_t = o_t \mid M_{1:t-1} O_{1:t-1} = \bar{m}_{1:t-1} \bar{o}_{1:t-1}, M_t = m_t \right )  \\
& \stackrel{(2)}{=} & \prod_{t \in [N]} \PP \left ( O_t =  o_t \mid O_{1:t-1} = \bar{o}_{1:t-1}, M_t = m_{t}  \right )  \\
& \stackrel{(3)}{=} & \prod_{t \in [N]} \PP \left ( X^{\obs}_{t} = o_{t} \mid X^{\obs}_{1:t-1} = \bar{o}_{1:t-1}, M_t = m_{t}  \right )  \\
& \stackrel{(4)}{=} & \prod_{\substack{t \in [N]} } \PP \left ( X^{\obs}_{t} = o_{t} \mid  X^{\obs}_{1:t-1} = \bar{o}_{1:t-1} \right ) \\
& \stackrel{(5)}{=} & \PP \left ( X^{\obs}_{1:N}  = \bar{o}_{1:N} \right ).
\end{eqnarray*}
where the equation (1) is by the definition of the observation process $f(\cdot)$; (2) reduces the redundant missingness information of $\bar{m}_{1:t-1}$ (as $\bar{o}_{1:t-1}$ has already contained the missingness information by Equation \ref{eq:observationProcess}); (3) follows as the missing values are wildcards;  the equation (4) can be established by only considering the cases $(m_t, o_t) = (1, \emptyset)$ and $(m_t, o_t) = (0, x_t)$, as otherwise $\PP(M_{1:N}O_{1:N} =  \bar{m}_{1:N} \bar{o}_{1:N}) = 0$ by Equation \ref{eq:combinedProcess} and \ref{eq:observationProcess}, violating the assumption of the proposition. First assume $(m_t, o_t) = (1, \emptyset)$, then
\[ \PP \left ( X^{\obs}_{t} = \emptyset \mid  X^{\obs}_{1:t-1} = \bar{o}_{1:t-1}, M_t = 1 \right ) = 1 = \PP \left ( X^{\obs}_{t} = \emptyset \mid  X^{\obs}_{1:t-1} = \bar{o}_{1:t-1} \right ).\] Next assume $(m_t, o_t) = (0, x_t)$. This means \[ \PP \left ( X^{\obs}_{t} = x_t \mid  X^{\obs}_{1:t-1} = \bar{o}_{1:t-1}, M_t = 0 \right ) = \PP \left ( X^{\obs}_{t} = x_t \mid  X^{\obs}_{1:t-1} = \bar{o}_{1:t-1} \right )\] as showed in Lemma \ref{lemma:AMSAR}. Hence (4) is true for both of the cases.  Equation (5) is by the general product rule of conditional probabilities.
\end{proof}

\begin{cor} \label{cor:amsObs}
Under the same assumption of Proposition \ref{prop:probObs}, additionally assuming that the underlying stochastic process $(X_t)_{t \in \mathbb{N}}$ is asymptotically stationary with stationary probability measure $\bar{\PP}$, the following equation holds:
\[ \bar{f} (\bar{m}_{1:N} \bar{o}_{1:N}) \coloneqq  \lim_{j \to \infty}  f (M_{j:j+N-1}O_{j:j+N-1} = \bar{m}_{1:N} \bar{o}_{1:N})  = \bar{\PP}(X^{\obs}_{1:N} = \bar{o}_{1:N}). \]
\end{cor}

\begin{proof}
Repeating the argument in Proposition \ref{prop:probObs}, it is clear that

\[ f (M_{j:j+N-1}O_{j:j+N-1} = \bar{m}_{1:N} \bar{o}_{1:N} )  = \PP(X^{\obs}_{j:N+j-1} = \bar{o}_{1:N}) \] 
for all $j$. Hence
\begin{eqnarray*}
\lim_{j \to \infty}  f (M_{j:j+N-1}O_{j:j+N-1} = \bar{m}_{1:N} \bar{o}_{1:N} )  & = & \lim_{j \to \infty} \PP(X^{\obs}_{j:N+j-1} = \bar{o}_{1:N}) \\
& = &  \bar{\PP}(X^{\obs}_{1:N} = \bar{o}_{1:N})
\end{eqnarray*}
where the last equation is by the assumption that the underlying stochastic process is asymptotically stationary.
\end{proof}

\begin{cor} \label{cor:probObs}
Under the same assumption of Corollary \ref{cor:amsObs}, the following equation holds:
\[\bar{f}(0x_1 \cdots 0 x_N) = \bar{\PP}( X_{1:N} = \bar{x}_{1:N} ). \]
\end{cor}

\begin{proof}
This directly follows from Corollary \ref{cor:amsObs} by letting $\bar{m}_{1:N} \bar{o}_{1:N} = 0x_1 \cdots 0 x_N$.
\end{proof}

\subsection{Spectral Learning for OOMs from data containing missing values}

Recall that, given a sequence (or sequences) of missingness-observation pairs $\bar{m}_{1:N}\bar{o}_{1:N}$ for some $N$, our goal is to learn a model for $(X_t)_{t \in \mathbb{N}}$ that describes the underlying stochastic process which does not contain missing values. \citet{Thon2017} observed that, if the underlying stochastic process $(X_t)_{t \in \mathbb{N}}$ can be described by an OOM $\Mcal = (\sigma, \{ \tau_{x} \}_{x \in \Sigma_{\Ocal}}, \omega_{\epsilon})$, then the observation process $f$ can be described by an IO-OOM $\Mcal' =  (\sigma',  \{\tau'_{m,o}\}, \omega'_\epsilon)$ in terms of the OOM parameters:
\begin{equation} \label{eq:AugmentOOM}
\sigma'  = \sigma, \quad   \omega'_\epsilon =  \omega_\epsilon, \quad \tau'_{0, \emptyset}= \mathbf{0}, \quad \forall  x \in \Sigma: \quad  \tau'_{0,x}  =  \tau_x, \quad  \tau'_{1, \emptyset} = \sum_{x \in \Sigma_{\Ocal}} \tau_x, \quad \tau'_{1, x} =  \mathbf{0}.
\end{equation}

Thus, in light of Proposition \ref{prop:probObs}, to learn an OOM $\hat{\Mcal} = (\hat{\sigma}, \{ \hat{\tau_{x}} \}_{x \in \Sigma_{\Ocal}}, \hat{\omega}_{\epsilon}) $ which approximates the OOM $\Mcal$ of the underlying stochastic process, we can first learn the IO-OOM $\hat{\Mcal'} = (\hat{\sigma'},  \{\hat{\tau'}_{m,o}\}, \hat{\omega'}_\epsilon)$ that approximtes $\Mcal'$ using the training data of missingness-observation sequences, and then reduce $\hat{\Mcal'}$ to $\hat{\Mcal}$ by reusing the observable operators with missingness $0$ and discarding other observable operators. More concretely, Thon proposed the following Algorithm \ref{alg:SpectralLearningMissingValues}, which takes missingness-observation sequences as input and estimates an OOM which describes the underlying stochastic process as output.

\begin{algorithm}[H]
\caption{ Spectral Learning for OOMs from data containing missing values. \label{alg:SpectralLearningMissingValues}}

\begin{algorithmic}[1]

\INPUT $M$ trajectories of observation sequences paired with $M$ trajectories of missingness indicator sequences of length $N$: $\{\bar{m}_{1:N}^{[j]} \bar{o}_{1:N}^{[j]}\}_{j = 1}^M$.

\PARAMETER (i) $d$: dimension of the IO-OOM and OOM, (ii) $Q = \{\bar{q}_j\}_{j = 1}^{D_1}$ and $C = \{\bar{c}_i\}_{i=1}^{D_2}$: indicative and characteristics sequences, where $q_j$, $c_i \in \Sigma^{\ast}$ for all $j$ and $i$, and (iii) $\hat{f}$: the frequency estimator for input-output sequences. 
\State Assemble estimates $\hat{F}_{C, Q} = [\hat{f} (\bar{q}\bar{c})]_{\bar{c} \in C, \bar{q} \in Q}, \hat{F}_{mo C, Q} = [\hat{f} (\bar{q}mo \bar{c})]_{\bar{c} \in C, \bar{q} \in Q}, \hat{F}^\top_{C} = [\hat{f}(\bar{c})]_{\bar{c} \in C}^{\top},\text{~and~}\hat{F}_{Q} = [\hat{f}(\bar{q})]_{\bar{q} \in Q}$ from $\{\bar{m}_{1:N}^{[j]} \bar{o}_{1:N}^{[j]}\}_{j = 1}^M$ using the frequency estimator $\hat{f} (\cdot)$. \label{algorithmStep:assembleEstimates}
\State Compute the $d$-truncated SVD of $\hat{F}_{C, Q}$: $ \hat{U}_d \hat{S}_d \hat{V}_d^{\top} \approx \hat{F}_{C, Q}$.

\State Compute $\hat{\Mcal'} = (\hat{\sigma'}, \hat{\tau'}_{m, o}, \hat{\omega'}_{\epsilon})$:\label{algorithmStep:IOOOMresult}
\begin{eqnarray*}
\hat{\sigma'} & = & \hat{F}_Q (\hat{U}_d^\top \hat{F}_{C,Q})^{\dagger} \\
\hat{\omega'}_{\epsilon} & = & \hat{U}_d^\top \hat{F}_{C}^\top \\
\hat{\tau'}_{m, o} & = & \hat{U}_d^\top \hat{F}_{mo C, Q} (\hat{U}_d^\top \hat{F}_{C, Q})^{\dagger},~\forall m o \in (\{0,1\} \times \Sigma_{\Ocal} \cup \{\emptyset\}). 
\end{eqnarray*}

\State Reduce the IO-OOM to OOM by defining $\hat{\Mcal} = (\hat{\sigma}, \hat{\tau}_{m, o}, \hat{\omega}_{\epsilon})$: \label{algorithmStep:OOMresult}
\begin{eqnarray*}
\hat{\sigma} & = & \hat{\sigma'} \\
\hat{\omega}_{\epsilon} & = & \hat{\omega'}_{\epsilon} \\
\hat{\tau}_{x} & = & \hat{\tau'}_{m, x}~\forall m x \in (\{0\} \times \Sigma_{\Ocal}). 
\end{eqnarray*}

\OUTPUT A model $\hat{\Mcal} = (\hat{\sigma}, \{\hat{\tau}_x\}, \hat{\omega}_\epsilon)$ of the underlying stochastic process.


\end{algorithmic}
\end{algorithm}

\begin{rem} \label{rem:OOMparameters}
In practice, we do not need to compute the IO-OOM operators $\hat{\tau'}_{m, o}$ for $m o \in (\{1\} \times \Sigma_{\Ocal} \cup \{\emptyset\} ) \cup \{ (0, \emptyset)\}$ in step \ref{algorithmStep:IOOOMresult}, because only observable operators $\hat{\tau'}_{0, x}$ for all $x \in \Sigma_{\Ocal}$ are relevant to the algorithm output. 
\end{rem}

\subsection{Theoretical Analysis of a Frequency Estimator $\hat{f} (\cdot)$} \label{subsection:estimator}

It turns out that a well-designed frequency estimators $\hat{f} (\cdot)$ is crucial to achieve consistency of Algorithm \ref{alg:SpectralLearningMissingValues}.  As an AMSAR missingness can be seen as a non-blind policy if we interpret a missingess-observation process as an input-output process, at a first glance, we can simply re-use an off-the-shelf frequency estimator for input-output processes with non-blind and unknown policies (e.g., \citep{Bowling2006}). A closer look, however, reveals that directly using such estimators is not entirely optimal for at least three reasons. For one, by default, these estimators treat all system outputs as genuine symbols, which means that the missing value symbols $\emptyset$ will not be treated as wildcards; secondly, to derive consistency using these estimators, we need to impose assumptions (e.g., asymptotic stationarity) on the missingness-observation processes -- we want to avoid this because such missingness-observation processes are artificial objects, which lack the transparency of underlying stochastic processes, the processes we eventually want to model; thirdly, these estimators take form as multi-step chained products of counting statistics (e.g.,  \citep[Equation 6]{Bowling2006}), giving rise to a higher computational cost than one customarily expects for frequency estimators for stochastic processes.


We now introduce and analyze a new frequency estimator, which addresses the above issues. For a given sequence of observations $\bar{z}_{1:k} \in (\Sigma_{\Ocal} \cup \{ \emptyset \})^{k}$, we first define an indicator function $\mathds{1}_{\bar{z}_{1:k}}^{\obs} (\cdot)$, which takes a (infinite-length) sequence of observations $\bar{o} = o_1 o_2 o_3 \cdots \in (\Sigma_{\Ocal} \cup \{ \emptyset \})^{\NN}$, compares its initial $k$-length sub-sequence $\bar{o}_{1:k}$ with $\bar{z}_{1:k}$, and then produces an integer $1$ or $0$, depending on whether the sequence $\bar{o}_{1:k}$ is identical to the sequence $\bar{z}_{1:k}$ up to the entries that are missing in $\bar{z}_{1:k}$. More precisely, we define

\begin{align*}
\mathds{1}_{\bar{z}_{1:k}}^{\obs}: (\Sigma_{\Ocal} \cup \{ \emptyset \})^{\NN} &\rightarrow \{0,1 \}, \\
 \bar{o} & \mapsto   \begin{cases}
1& \text{if~} o_i = z_i \text{~whenever~} z_i \neq \emptyset, i \in [k]. \\
0& \text{else}.
\end{cases} 
\end{align*}

To count the number of appearances of a sequence $\bar{z}_{1:k}$ as a subsequence of a finite sequence $\bar{o}_{1:n}$ for some $n, k \in \NN, n \geq k$, we require the first $n$ symbols of $\bar{o} \in (\Sigma \cup \{\emptyset\})^{\NN}$ to be $\bar{o}_{1:n}$ and write

\begin{equation} \label{eq:estimator1}
 {\#}_{\bar{z}_{1:k}}^{\obs} (\bar{o}_{1:n}) \coloneqq \sum_{i = 0}^{n-k} \mathds{1}_{\bar{z}_{1:k}}^{\obs} (T^i ( \bar{o})).
 \end{equation}
 As a concrete example, consider the sequence $\bar{o}_{1:5} = \emptyset \texttt{bab} \emptyset$ and $\bar{z}_{1:2} =  \emptyset \texttt{b}$. In this case, 

\begin{alignat*}{2}
    \mathds{1}_{\bar{z}_{1:2}}^{\obs} ( \bar{o}) &= \mathds{1}_{\emptyset \texttt{b}}^{\obs} ( \emptyset \texttt{bab} \emptyset \cdots) && = 1,\\
 \mathds{1}_{\bar{z}_{1:2}}^{\obs} ( T(\bar{o})) & =   \mathds{1}_{\emptyset \texttt{b}}^{\obs} ( \texttt{bab} \emptyset \cdots) && = 0,\\
  \mathds{1}_{\bar{z}_{1:2}}^{\obs} ( T^2(\bar{o})) & =   \mathds{1}_{\emptyset \texttt{b}}^{\obs} ( \texttt{ab} \emptyset \cdots) && = 1,\\
    \mathds{1}_{\bar{z}_{1:2}}^{\obs} ( T^3(\bar{o})) &  =  \mathds{1}_{\emptyset \texttt{b}}^{\obs} ( \texttt{b} \emptyset \cdots) && = 0,
    \end{alignat*}
and therefore $ {\#}_{\bar{z}_{1:2} }^{\obs} ( \bar{o}_{1:5}) = {\#}_{ \emptyset \texttt{b}}^{\obs} (\emptyset \texttt{bab} \emptyset) = \mathds{1}_{\emptyset \texttt{b}}^{\obs} ( \emptyset \texttt{bab} \emptyset \cdots)+ \mathds{1}_{\emptyset \texttt{b}}^{\obs} ( \texttt{bab} \emptyset \cdots) + \mathds{1}_{\emptyset \texttt{b}}^{\obs} ( \texttt{ab} \emptyset \cdots) + \mathds{1}_{\emptyset \texttt{b}}^{\obs} ( \texttt{b} \emptyset \cdots) =   1 + 0 + 1 + 0 = 2$.


\begin{prop} \label{prop:estimatorConsistency}
Let the infinite sequence $\bar{x} = \otimes_{t \in \NN} X_t (\omega)$ for some $\omega \in \Omega$ be any realization of an asymptotic stationary and ergodic stochastic process $(\Omega, \Fcal, \PP, (X_t)_{t \in \NN})$. Let $\bar{x}_{1:n}, n \in \NN$, be an initial subsequence of $\bar{x}$. Let $\bar{m}_{1:n} \bar{o}_{1:n}$ be the corresponding initial missingness-observation sequence which is obtained by corrupting the underlying sequence $\bar{x}_{1:n}$ with an AMSAR missingness. For any missingness-observation sequence $\bar{u}_{1:k} \bar{z}_{1:k}$, setting $\hat{f}_n (\bar{u}_{1:k} \bar{z}_{1:k} ) = \frac{{\#}_{\bar{z}_{1:k}}^{\obs} (\bar{o}_{1:n} )}{n - k + 1}$, where $n \geq k$, the equation 

\[ \lim_{n \geq k, n \to \infty} \hat{f}_n (\bar{u}_{1:k} \bar{z}_{1:k} ) = \bar{f} (\bar{u}_{1:k} \bar{z}_{1:k}) \]
holds $\PP$-almost surely, where $\bar{f}$ is the stationary observation process defined in Corollary \ref{cor:amsObs}.
\end{prop}

\begin{proof} 

Let $\bar{o} \in (\Sigma_{\Ocal} \cup \{ \emptyset \} )^{\NN}$ be a sequence and let $\bar{o}_{1:n}$ be an initial subsequence of $\bar{o}$ for all $n \in \NN$.
\begin{eqnarray*}
  \lim_{n \to \infty} \hat{f}_n (\bar{u}_{1:k} \bar{z}_{1:k} ) & = & \lim_{n \to \infty}  \frac{{\#}_{\bar{z}_{1:k}}^{\obs} (\bar{o}_{1:n} )}{n - k + 1} \\
   & = & \lim_{n \to \infty} \frac{1}{n - k+1} \sum_{i = 0}^{n - k} \mathds{1}_{\bar{z}_{1:k}}^{\obs} \left (T^i ( \bar{o}) \right ) \\
         & \stackrel{(\ast)}{=} & \bar{\EE} \left [ \mathds{1}_{\bar{z}_{1:k}}^{\obs} (\bar{o})  \right ] \\
        & = & \bar{\PP}  ( \{\bar{o}: o_i = z_i \text{~whenever~} z_i \neq \emptyset, i \in [k]\}) \\
        & = & \bar{\PP}  (  X^{\obs}_{1:k} = \bar{z}_{1:k}) \\
      & \stackrel{(\ast \ast)}{=} & \bar{f} (\bar{u}_{1:k} \bar{z}_{1:k}).
\end{eqnarray*}
where the equation $(\ast)$ holds almost surely by Proposition \ref{prop:Ergodic} and the equation $(\ast \ast)$ is by Corollary \ref{cor:probObs}.

\begin{rem} \label{rem:Estimator} The originally proposed frequency estimator that acknowledges the missing values \citep[Section 7.2]{Thon2017} was formulated in a multi-step chained product form, similar to those for input-output processes \citep{Bowling2006}. We showed that, however, the factored multiplication is not needed. In this sense, our proposed estimator is a simplification of that of \citep[Section 7.2]{Thon2017}.
\end{rem}

\end{proof}

\subsection{Theoretical Analysis of the Algorithm \ref{alg:SpectralLearningMissingValues}} \label{sec:TheoreticalAnalysis}

In this section we analyze the consistency of the Algorithm \ref{alg:SpectralLearningMissingValues}, which is the main result of this work. To assure that Algorithm \ref{alg:SpectralLearningMissingValues} is consistent, \citet[Section 7.1]{Thon2017} assumed that (i) the underlying stochastic process $(X_t)_{t \in \mathbb{N}}$ is stationary and ergodic, (ii) the missingness process $(M_t)_{t \in \NN}$ is stationary and ergodic, (iii) at least one of $(X_t)_{t \in \mathbb{N}}$ or $(M_t)_{t \in \NN}$ is weakly mixing, (iv) $(M_t)_{t \in \NN}$ is indepedent of $(X_t)_{t \in \mathbb{N}}$, and (v) the missingness is strictly uncertain\footnote{The missingness is said to be \emph{strictly uncertain} if $\pi (\bar{m}_{1:N} \bar{o}_{1:N}) > 0$ for all $\bar{m}_{1:N} \bar{o}_{1:N}$, where $\pi (\bar{m}_{1:N} \bar{o}_{1:N})$ is defined in Equation \ref{eq:combinedProcess}.}. We now analyze Algorithm \ref{alg:SpectralLearningMissingValues} with substantial relaxations: We relax the requirement of the stationary and ergodic stochastic process $(X_t)_{t \in \mathbb{N}}$ of assumption (i) to that of an asymptotically stationary and ergodic process; we drop the assumptions (ii), (iii), and (v) altogether; we relax the independence missingness assumption in (iv) to that of \emph{always missing sequentially at random}, where the missingness values of the latter \emph{can} depend on the previous observations. More precisely, we show that the learned OOM from data containing missing values is consistent up to an initial OOM state under the assumptions listed below.

\begin{assumption} \label{assump:stochasticProcess}
The underlying stochastic process $(X_t)_{t \in \mathbb{N}}$ is ergodic, asymptotically stationary, and modeled by a $d$-dimensional OOM $\Mcal = (\sigma, \{\tau_{x}\}_{x \in \Sigma_{\Ocal}}, \omega_{\epsilon})$. Let $\{ \bar{x}_{1:N}^{[j]}\}_{j = 1}^M$ be $M$ initial samples of a stochastic process $(X_t)_{t \in \mathbb{N}}$, each with length $N$ and indexed by $j$. An AMSAR missingness transforms $\{ \bar{x}_{1:N}^{[j]}\}_{j = 1}^M$ into the training data $\{ \bar{m}_{1:N}^{[j]}\bar{o}_{1:N}^{[j]} \}_{j = 1}^M$ of Algorithm \ref{alg:SpectralLearningMissingValues}.\end{assumption}

\begin{assumption} \label{assump:estimator}
The estimator $\hat{f}(\cdot)$ in Proposition \ref{prop:estimatorConsistency} is used to assemble the estimated Hankel matrices $\hat{F}_{C, Q}, \hat{F}_{mo C, Q}, \hat{F}^\top_{C}, \hat{F}_{Q}$ from $\{\bar{m}_{1:N}^{[j]} \bar{o}_{1:N}^{[j]}\}_{j = 1}^M$.
\end{assumption}

\begin{assumption} \label{assump:modelRank}
The set of characteristics sequences $C$ and that of indicative sequences $Q$ are specified such that the rank of $\plim  \hat{F}_{C, Q}$ is no less than $d$.
\end{assumption}

\begin{thm} \label{thmAsymptoticallyCorrect}
Under the Assumptions \ref{assump:stochasticProcess}-\ref{assump:modelRank} above, for every estimated OOM $\hat{\Mcal} =  (\hat{\sigma}, \{ \hat{\tau}_{x}\}_{x \in \Sigma_{\Ocal}}, \hat{\omega}_{\epsilon})$ produced by Algorithm \ref{alg:SpectralLearningMissingValues}, there exists an equivalent OOM $\tilde{\Mcal} = (\tilde{\sigma}, \{\tilde{\tau}_{x}\}_{x \in \Sigma_{\Ocal}}, \tilde{\omega}_{\epsilon})$, such that $\tilde{\Mcal}$ is consistent in the sense that

\[ \tilde{\sigma}  \xrightarrow[]{p} \sigma, \tilde{\tau}_{x} \xrightarrow[]{p}  \tau_{x}~\forall x \in \Sigma_{\Ocal},\]
as size of the training data $\{ \bar{m}_{1:N}^{[j]}\bar{o}_{1:N}^{[j]} \}_{j = 1}^M$ approaches infinity with $N \to \infty$. That is, sequences of $\tilde{\sigma'}$ and $\tilde{\tau'}_{mo}$, each estimated from initial training  sequences with increasing length, will respectively converge (in the entry-wise sense) to the vector $\sigma'$ and matrix $\tau'_{mo}$ in the probability limit.
\end{thm}

\begin{proof}
Without loss of generality, it is enough to assume that the training data has a single trajectory $\bar{m} \bar{o}$, i.e., $M = 1$ and $N \to \infty$. According to Assumption \ref{assump:stochasticProcess}, the underlying stochastic process is modeled by a $d$-dimensional underlying OOM $\Mcal = (\sigma, \{ \tau_{x}\}_{x \in \Sigma_{\Ocal}}, \omega_{\epsilon})$. We can augment the underlying OOM $\Mcal$ to an IO-OOM $\Mcal'  = (\sigma', \{\tau'_{m,o}\}, \omega'_{\epsilon})$ using Equation \ref{eq:AugmentOOM}. To prove the theorem, it is enough to show that there exists an IO-OOM $\tilde{\Mcal'} = (\tilde{\sigma'}, \tilde{\tau'}_{m,o}, \tilde{\omega'}_\epsilon)$ which is equivalent to $\hat{\Mcal'} =  (\hat{\sigma'}, \{\hat{\tau'}_{m,o}\}, \hat{\omega'}_{\epsilon})$ of Algorithm \ref{alg:SpectralLearningMissingValues}, such that $\tilde{\Mcal'}$ satisfies 

\[ \tilde{\sigma'}  \xrightarrow[]{p} \sigma', \tilde{\tau'}_{m,o} \xrightarrow[]{p}  \tau'_{m,o}~\forall~m o \in (\{0\} \times \Sigma_{\Ocal}).\]

Define 
\begin{eqnarray}
 \omega'_{\ast} & \coloneqq & \text{The eigenvector~that~corresponds~to~the~eigenvalue~1~of~} \tau'_{0, \Sigma_{\Ocal}}   \\
 \Phi_Q & \coloneqq & \left [ \tau'_{\bar{q}} \omega'_{\ast} \right]_{\bar{q} \in Q} \in \RR^{d \times D_1}\\
\Pi_C & \coloneqq & \left[ (\sigma \tau_{\bar{c}})^\top \right ]_{\bar{c} \in C}^\top \in \RR^{D_2 \times d},
\end{eqnarray}
where $D_1$ and $D_2$ are the cardinalities of the sets of indicative and characteristics sequences. Note that the existence and uniqueness of $\omega'_{\ast}$ is guaranteed by as the stochastic process is assumed to be ergodic in Assumption \ref{assump:stochasticProcess} (\citep[Theorem 5.2]{Schonhuth2009}).

By Assumption \ref{assump:estimator} and proposition \ref{prop:estimatorConsistency},

\begin{eqnarray} \label{eq:HankelMatrices1}
\hat{F}_Q & \xrightarrow[]{p} &  \sigma \Phi_Q \in \RR^{1 \times D_1} \\
\hat{F}_C & \xrightarrow[]{p} & \Pi_{C}   \omega'_{\ast} \in \RR^{1 \times D_2} \\
\hat{F}_{C, Q} & \xrightarrow[]{p} &  \Pi_C \Phi_Q \in \RR^{D_2 \times D_1} \\ 
\hat{F}_{mo C, Q} & \xrightarrow[]{p} &  \Pi_C \tau_{mo} \Phi_Q  \in \RR^{D_2 \times D_1} \label{eq:HankelMatrices4}
\end{eqnarray}

Note that 

\[d \leq \rank ( \plim \hat{F}_{C, Q} ) = \rank (\Pi_C \Phi_Q) \leq \min \left ( \rank (\Pi_c), \rank(\Phi_Q) \right ) \leq d. \]
where the first $\leq$ is by Assumption \ref{assump:modelRank} and the last $\leq$ is by the fact that $\Pi_C \in \RR^{D_2 \times d}, \Phi_Q \in \RR^{d \times D_1}$. This implies 
\[\rank(\Pi_C \Phi_Q) =  \rank (\Pi_C) = \rank(\Phi_Q) = d. \]


As $\rank(\Pi_C \Phi_Q) = d$, $\Pi_C \Phi_Q$ can be decomposed exactly with $d$-truncated SVD. In particular, let $U_d S_d V_d^\top = \Pi_C \Phi_Q$. We claim that $U_d^\top \Pi_C \in \RR^{d \times d}$ is invertible. To show this, consider 

 \begin{eqnarray}
  U_d S_d V_d^\top =  \Pi_C \Phi_Q \\
  \Rightarrow  S_d V_d^\top = U_d^\top \Pi_C \Phi_Q 
\end{eqnarray}

As  $S_d$ is an invertible matrix,
\begin{eqnarray}
& ~& \rank(S_d V_d^\top) = \rank(V_d^\top) = d \\
& \Rightarrow &  \rank(U_d^\top \Pi_C \Phi_Q) = \rank(S_d V_d^\top) = d \\
&  \Rightarrow &  d = \rank(U_d^\top \Pi_C \Phi_Q) \leq \min \left (\rank(U_d^\top \Pi_C), \rank(\Phi_Q) \right ) \\
&  \Rightarrow &   \rank(U_d^\top \Pi_C) \geq d.
\end{eqnarray}

As $U_d^\top \Pi_C \in \RR^{d \times d}$, this implies $\rank \left ( U_d^\top \Pi_C \right ) = d$, and therefore $U_d^\top \Pi_C$ is invertible. Now we argue that $\plim \left ( \hat{U}_d^\top \Pi_C \right )$ is invertible, too. Indeed, we can write $\plim \hat{U}_d = A U_d$, where $A$ is some permutation matrix\footnote{Note that a SVD is only unique up to permutations of singular vectors and singular values. See, e.g., \cite[pg. 99]{Dasgupta2006}.}. Since such a permutation matrix $A$ is invertible, it is clear that $ \rank \left ( \plim \hat{U}_d^\top \Pi_C \right) = \rank (A U_d^\top \Pi_C) = \rank (U_d^\top \Pi_C)  = d$, and therefore $\plim \left ( \hat{U}_d^\top \Pi_C \right )$ is invertible. 

Define $\rho \coloneqq \plim \hat{U}_d^\top \Pi_C =  A U_d^\top \Pi_C $ and let
\begin{eqnarray}
\tilde{\sigma'} & \coloneqq & \hat{\sigma'} \rho  \\
 \tilde{\tau'}_{m,o} & \coloneqq & \rho^{-1}~\hat{\tau'}_{mo}~\rho   \\
 \tilde{\omega'}_{\epsilon} & \coloneqq & \rho^{-1}~\omega'_{\epsilon}. 
\end{eqnarray}

That is, $\tilde{\Mcal'}$ and $\hat{\Mcal'}$ are the same up to a similarity transformation induced by $\rho$. We now show that 

\[ \tilde{\sigma'}  \xrightarrow[]{p} \sigma', \tilde{\tau'}_{mo} \xrightarrow[]{p}  \tau'_{mo}~\forall mo.\]

Observe that

\begin{equation}  \label{eq:sigmaCorrectness}
\begin{split}
\tilde{\sigma'} & = {\color{Magenta} \hat{\sigma'}} {\color{LimeGreen} \rho} \\
 & = {\color{Magenta}{\hat{F}_Q (\hat{U}_d^\top \hat{F}_{C,Q})^{-1}}} {\color{LimeGreen} (A U_d^\top \Pi_C)} \\ 
 & \xrightarrow[]{p} (\sigma \Phi_Q) (A U_d^\top \Pi_C \Phi_Q)^{-1} (A U_d^\top \Pi_C) \\
     & = (\sigma \Phi_Q) \Phi_Q^{\dagger} \\
    & =  \sigma' \\    
\end{split}
\end{equation}

and that

\begin{equation} \label{eq:tauCorrectness}
\begin{split}
\tilde{\tau'}_{mo} & = {\color{NavyBlue} \rho^{-1}}~{\color{BrickRed} \hat{\tau'}_{mo}}~{\color{NavyBlue} \rho} \\
 & = {\color{NavyBlue} (A U_d^\top \Pi_C)^{-1}}~{\color{BrickRed} \hat{\tau'}_{mo}}~{\color{NavyBlue}(A U_d^\top \Pi_C)} \\
  & = {\color{NavyBlue}(A U_d^\top \Pi_C)^{-1}}~{\color{BrickRed}\hat{U}_d^\top \hat{F}_{mo C, Q} (\hat{U}_d^\top \hat{F}_{C, Q})^{-1}}~{\color{NavyBlue}(A U_d^\top \Pi_C)} \\
  & \xrightarrow[]{p} (A U_d^\top \Pi_C)^{-1}~A U_d^\top \Pi_C \tau_{mo} \Phi_Q (A U_d^\top \Pi_C \Phi_Q)^{-1}~(A U_d^\top \Pi_C) \\
  & = \Pi_C^{\dagger}~\Pi_C \tau_{mo} \Phi_Q (\Pi_C \Phi_Q)^{\dagger}~\Pi_C \\
    & = \tau'_{mo}
    \end{split}
\end{equation}

and we are done. 

We point out that in the above derivation we have implicitly assumed that $\hat{U}_d^\top \hat{F}_{C,Q} \xrightarrow[]{p} A U_d^\top  \Pi_C \Phi_Q$ implies $(\hat{U}_d^\top \hat{F}_{C,Q})^{-1} \xrightarrow[]{p} (A U_d^\top  \Pi_C \Phi_Q)^{-1}$. This holds thanks to the Continuous Mapping Theorem \citep{Mann1943} because the matrix inverse is a continuous mapping for the full rank square matrix $A U_d^\top  \Pi_C \Phi_Q$.

\end{proof}

\section{Empirical Results} \label{sec:empiricalResults}

In this section, we empirically evaluate our proposed method and compare it with several baseline methods on synthetic data and real-world data.

\subsection{Evaluation Metric}

We empirically evaluate our proposed method with synthetic experiments and real-world experiments. We compare it against three baseline methods. The first baseline method is the EM-based Baum-Welch algorithm that learns HMMs from data containing missing values \citep{Yeh2012}. We will refer to this method as \emph{missing value HMM}. Another reasonable approach is to first truncate the training data that contain missing values into short trajectories that are free from missing values, and learn the model parameters based on the ensemble of short trajectories, leading to the resulted models \emph{short trajectory HMM} and \emph{short trajectory OOM}.

The following settings were used in the experiments. To train short trajectory HMMs we use the public-domain implementation of discrete observation Hidden Markov Model provided by Zoubin Ghahramani\footnote{\texttt{http://mlg.eng.cam.ac.uk/zoubin/software/dhmm.tar.gz}}. The maximum number of cycles of Baum-Welch was set to be 100, and the termination tolerance was set to be 0.0001. For the missing value HMM, we modify Ghahramani's implementation according to the recommendation of \citep{Yeh2012} while keeping algorithm parameters the same. To train short trajectory OOMs and missing value OOMs, we use indicative and characteristics sequences of length 3 in all experiments.

For synthetic experiments, given a learned model $\mathcal{\hat{M}}$, the true model $\mathcal{M}$, and the testing dataset $D$, we evaluate the model learning results using the metric \emph{log of average one-step prediction error} (LAOSPE):
\begin{align*}
 \text{LAOSPE} (\hat{\mathcal{M}}; \mathcal{M}, D) =  \log_2 \left [ \frac{1}{|D|} \sum_{\bar{x} \in D} \frac{1}{ |\bar{x}|} \sum_{t = 1}^{|\bar{x}|} \frac{1}{ |\Sigma_{\Ocal}|} \sum_{ o_t \in \Sigma_{\Ocal}}  \left( \mathbb{P}_{\hat{\mathcal{M}}} (o_t \mid \bar{x}_{1:t-1} ) \right. \right. \\  
 \left. \left. - \mathbb{P}_{\Mcal} \left ( o_t \mid \bar{x}_{1:t-1} \right ) \right )^2 \right ].
 \end{align*}

For  experiments with real-world data, the metric LAOSPE cannot be used since the true model $\Mcal$ is unknown. For this reason, we use the \emph{average negative log likelihood} (ANLL):
\begin{equation*}
 \text{ANLL} (\hat{\mathcal{M}}; D) = - \frac{1}{|D|} \sum_{\bar{x} \in D} \frac{1}{ |\bar{x}|} \log_2 (\mathbb{P}_{\hat{\mathcal{M}}} (\bar{x})).
  \end{equation*}

In practice, it might not be feasible to evaluate ANLL for all $\hat{\mathcal{M}}$ because the obtained OOM models might assign negative "probabilities" for some sequences \citep{Jaeger2006}.  For this reason, we first normalize the learned OOMs as specified in \citep[Appendix J]{Jaeger2006}.

\subsection{Experiment with synthetic HMM data}

\begin{figure}
\begin{centering}
\includegraphics[width=\columnwidth]{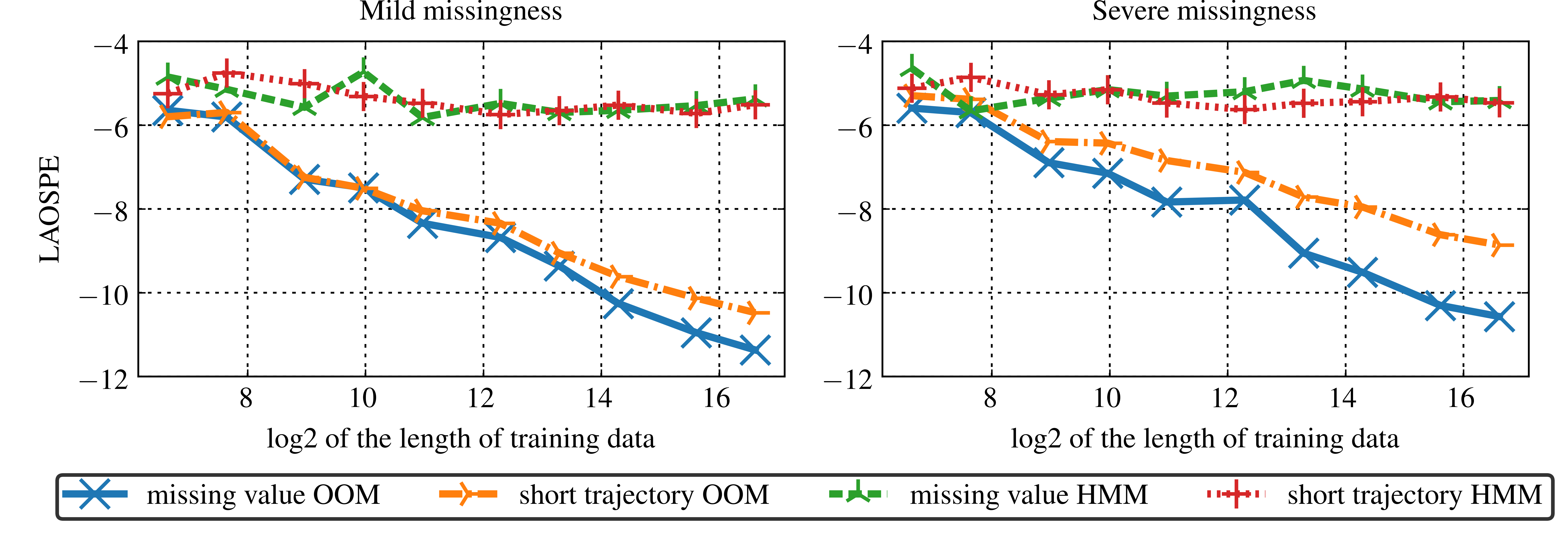}
\par\end{centering}
\caption{Testing results for models estimated based on mildly or severely corrupted ring-topology HMM data \label{fig:ring}. The LAOSPE score is calculated based on the training trajectory of length 100, 200, 500, 1000, 2000, 5000, 10000, 20000, 50000, and 100000 sequentially (note that the $x$-axis is $\log_2$ scaled). Upon using the training trajectory of length 100000, 8523 values are missing in the mild missingness scenario and 21244 values are missing for the severe missingness scenario.}
\end{figure}

We test our method on a synthetic HMM dataset similar to the ones used in \citep{Jiang2016} and \citep{Downey2017}. The data is randomly generated by a ring-topology HMMs with 20 latent states and 20 observations, where each latent state has at most two possible observations chosen randomly. The transition matrix follows a ring-topology, where each state can only transit to its two neighbors or to itself. All non-zero entries of the transition matrix, the emission matrix, and the initial state distribution are picked uniformly randomly from [0, 1) and then normalized. A training trajectory of length $10^6$ was sampled with the ring-topology HMM. To introduce artificial missing values in the training data, we consider two kinds of missingness: (i) ``mild missingness,'' for which we randomly choose 5 observations such that with probability 0.3 the immediate observation after any one of these five observations will be turned into a missing value symbol; and (ii) ``severe missingness,'' for which we randomly choose 10 observations such that with probability 0.5 the immediate observation after any one of these ten observations will be turned into a missing value symbol. It is clear that both of the severe and mild missingness are AMSAR. Moreover, in neither case, the missingness process is independent of the observation process. HMM and OOM models are then estimated based on the resulted training trajectory containing either mild or severe missing values. We additionally sample 10000 trajectories each with length 100 using the ring-topology HMM as the testing dataset, which does not contain missing values. Figure \ref{fig:ring} shows the testing results for all learned models. Apparently, the proposed missing value OOM method outperforms the three baseline methods. Compared with the baseline methods, missing value OOMs are particularly advantageous when the missingness is severe. The experiments also empirically demonstrate the asymptotic consistency of the two spectral methods: The spectral-method-yielded OOMs continuously improve the LAOSPE scores as the training trajectory becomes increasingly longer. The proposed method of missing value OOM has a faster empirical convergence rate compared with that of short trajectory OOM.

\subsection{Experiment with real-world geyser data}

We showcase the effectiveness of the proposed method using the ranger log data of the Old Faithful geyser, a cone geyser located in Yellowstone National Park, the United States, recorded by \citet{Stephens2010, Stephens2012}. As the ranger logs were usually kept during the working day and in the non-winter months, missing values are commonplace in the dataset. In particular, records at the weeks in November and March are mostly missing as the park personnel was on furlough \citep{Hartigan2013}. The goal of our experiment is to model the number of daily eruptions of Old Faithful geysers. For that purpose, we use the trajectory of daily eruptions of Old Faithful geysers from the year 2000-2007 as the training data, which contains the eruption records for 2922 days. Among 2922 days, the records for 666 days are missing. We ensemble the short trajectories that are free from missing values in the year 2008 - 2010 as the testing data. The experiment results are shown in Figure \ref{fig:geyser}. We see that the performance of EM methods are jittering, possible due to the fact that the training dataset size is small. In contrary, the spectral methods for short trajectory OOMs and missing values OOMs are more stable. In particular, the proposed missing value OOM method favorably yields lower ANLL compared to each of the alternatives. 

\begin{figure}
\begin{centering}
\includegraphics{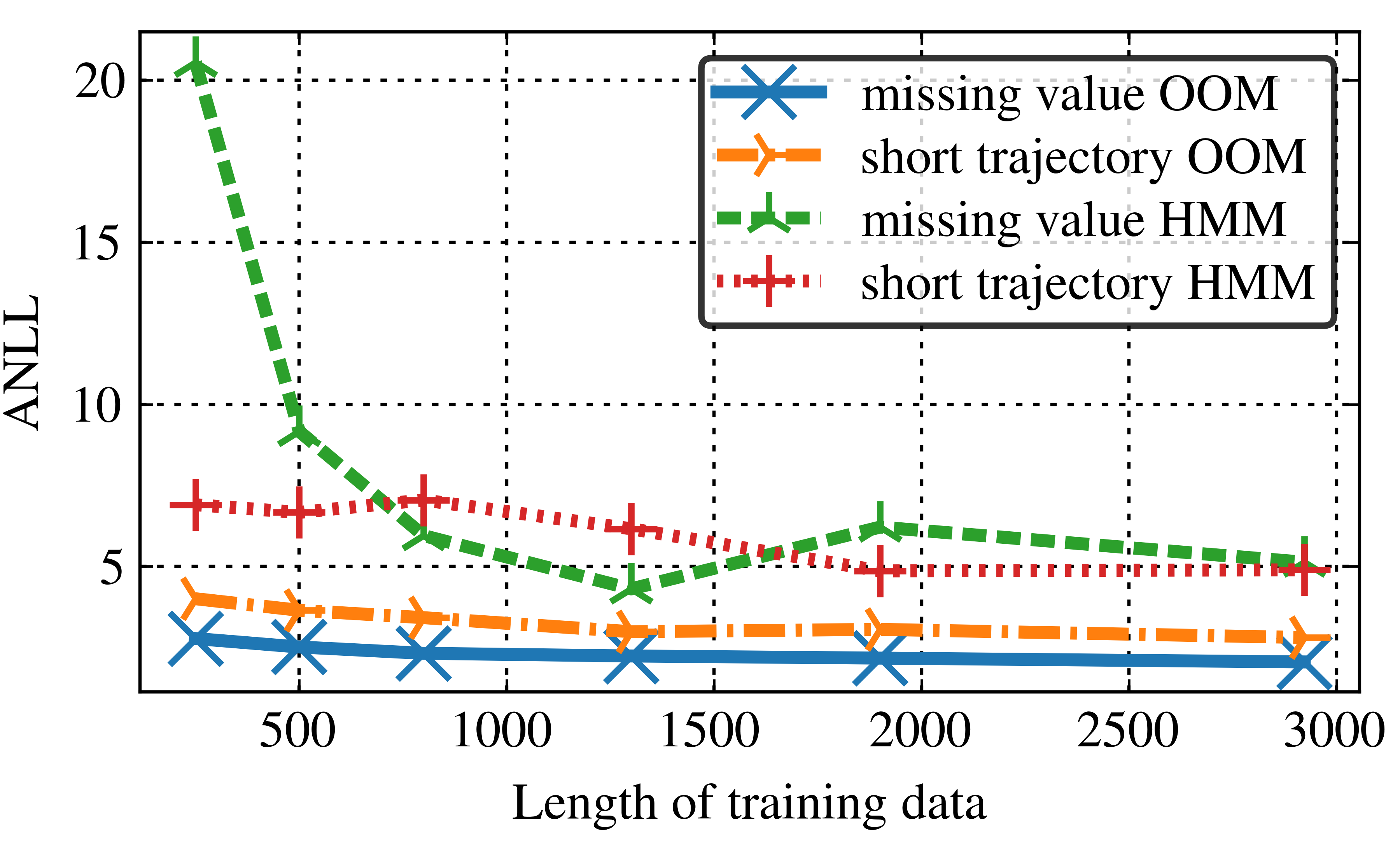}
\par\end{centering}
\caption{Testing results for models estimated from ranger log data of the Old Faithful geyser 2000-2007. The ANLL score is calculated based on the training trajectory of length 250, 500, 800, 1300, 1900, and 2922. Upon using the training trajectory of length 2922, 666 values are missing. \label{fig:geyser}}
\end{figure}

\section{Conclusion}
We presented a novel spectral method for learning OOMs from training data containing missing values originally proposed in \citep{Thon2017}. We analyzed sufficient conditions for achieving asymptotic consistency of the proposed algorithm and showed that such a set of conditions is practically relevant. By simulation, we demonstrated that the proposed method compares very favorably against previously used EM and spectral algorithms in a synthetic and a real-world dataset. 

There are several possibilities for future work. For one, it is tempting to investigate other spectral methods such as tensor decomposition methods to handle time series data containing missing values. The results reported in Section \ref{subsection:types}-\ref{subsection:estimator} for general stochastic processes can readily be reused for such purposes. Additionally, the current paper only considers real-time sequential systems. In many other scenarios such as gene modeling, however, there is no time structure in data, and therefore conditional independence assumptions such as AMSAR can no longer be safely made -- a problem left for future work.
\section*{Acknowledgement}
The author is grateful to Prof. Dr. Herbert Jaeger for his valuable suggestions throughout this project and for his careful proofreading of the final manuscript. The author thanks Dr. Michael Thon for helpful discussions. The author is supported by a Jacobs University Bremen Graduate Scholarship.

\newpage

\bibliography{/Users/liutianlin/Desktop/Academics/MINDS/OOM/oom_progress.bib}


\end{document}